\documentclass{article}
\usepackage{ijcai22}

\usepackage{times}
\usepackage{amsfonts}

\usepackage{soul}
\usepackage{url}
\usepackage[hidelinks]{hyperref}
\usepackage[utf8]{inputenc}
\usepackage[small]{caption}
\usepackage{graphicx}
\usepackage{amsmath}
\usepackage{booktabs}
\usepackage{color}
\newtheorem{definition}{Definition}
\newtheorem{prop}{Property}
\newtheorem{proof}{Proof}

\title{A Constraint Programming Approach to Weighted Isomorphic Mapping \\of Fragment-based Shape Signatures}

\author{
Thierry Petit$^1$\and
Randy J. Zauhar$^2$\and
\affiliations
$^1$Department of Mathematics, Physics and Statistics, University of the Sciences in Philadelphia.\\
$^2$Department of Chemistry and Biochemistry, University of the Sciences in Philadelphia.\\
\emails
tpetit@usciences.edu, rzauhar@usciences.edu
}

\begin{document}

\maketitle

\begin{abstract}
Fragment-based shape signature techniques have proven to be powerful tools for computer-aided drug design.
They allow scientists to search for target molecules with some similarity to a known active compound.
They do not require reference to the full underlying chemical structure, which is essential to deal with chemical databases containing millions of compounds.
However, finding the optimal match of a part of the fragmented compound can be time-consuming.
In this paper, we use constraint programming to solve this specific problem. It involves finding a weighted assignment of fragments subject to connectivity constraints. Our experiments demonstrate the practical relevance of our approach and open new perspectives, including generating multiple, diverse solutions. 

Our approach constitutes an original use of a constraint solver in a real time setting, where propagation allows to avoid an enumeration of weighted paths. The model must remain robust to the addition of constraints making some instances not tractable. This particular context requires the use of unusual criteria for the choice of the model: lightweight, standard\footnote{Alorithms provided as standard in the solvers of the literature, allowing an integration of the software in platforms where the programming language is imposed. } propagation algorithms, data structures without prohibitive constant cost. The objective is \underline{\bf not} to design new, complex algorithms to solve difficult instances. 
\end{abstract}

\section{Introduction}
Assessing shape similarity using signatures to identify compounds likely to be active against a given target is a state-of-the-art technique for computer-aided drug design~\cite{zau03,meek06}. 
The fragment-based implementation of Shape Signatures~\cite{zau13} has proven to enhance the selective power dramatically. Its effectiveness lies in the idea of partitioning molecules into fragments based on cycle systems or by custom speciﬁcation provided by the user.  The signatures are probability distributions stored as histograms generated with ray tracing. The method does not require reference to the detailed underlying chemical structure, which is essential for exploring large chemical databases.
The query molecule and the target molecule are both fragmented; chemical bonds (edges) link the parts (nodes). There are no cycles in the fragment-based decompositions~\cite{zau13}. Their comparison is performed by searching for an isomorphic mapping of the two fragmented compounds. The solution should minimize the dissimilarity between associated query and target fragment histograms. This is a weighted \emph{tree} isomorphism problem, much more suited to database search than the general subgraph isomorphism problem.

Despite its advantages, fragment-based signatures are, in practice, restricted to the search of complete query mapping into the target molecule, not just part of the query molecule. This limitation is mainly due to the nature of the data. Databases like PubChem (\url{https://pubchem.ncbi.nlm.nih.gov/}) and ZINC (\url{https://zinc.docking.org/}) contain tens of millions of compounds, which makes brute-force approaches too time-consuming when only part of the query molecule should be mapped on the target. Other interesting features such as finding (a few) maximally diverse solutions for each query-target candidate also require advanced optimization algorithms.

In this article, we formulate the ``partial query problem'' as a constraint network. Using constraint programming allows us to encode both the weighted assignment and connectivity requirements as modular components, making the model suited to further modifications for other end-user needs. This feature also preserves the remarkably easy use of the fragment-based technique. We only require that end-users supply a query structure and a few parameters to control the number and nature of solutions. 

Our study is the first step towards the integration of constraint programming for fragment-based computer-aided drug discovery. We use an open-source, problem-independent solver to investigate whether this technology can be suited to computer-aided drug discovery  objectives. Our experiments on real and simulated data demonstrate the significance of our approach. We show that large query and target molecules can be considered. Our model can generate multiple, diverse solutions.

The paper is organized as follows. We first recall some notions of constraint programming and discuss graph isomorphism solving tools.  In Section~\ref{sec:pb}, 
we formally state the problem and objectives.  
Section~\ref{sec:CP} describes the constraint model in detail. 
Our experiments (Section~\ref{sec:expes}) are split into three parts: real data, simulated larger data, and 
diverse solutions. We then conclude and discuss perspectives. 
\section{Background}
\subsection{Constraint networks}
A \emph{constraint network}~\cite{mon74} (also called \emph{constraint programming model}~\cite{tenyCPAIOR}) is defined over a set $X$ of \emph{domain variables}, a finite subset $D$ of $\mathbb{Z}$ called the \emph{domain union}, and a set of constraints $C$. 
A constraint $c \in C$ is a pair $\{var(c),rel(c)\}$, where $var(c)$, its \emph{scope}, is a subset of $X$, and $rel(c)$ is a relation that restricts the allowed
combinations of simultaneous value assignments for the variables in $var(c)$. 
Each variable $x \in X$ is defined by a domain unary constraint, that holds if and only if $x$ takes its value in $D(x) \subseteq D$. 
During the search for a solution, 
domains are modified by the solver, e.g., through a branch and bound scheme. 
The \emph{search strategy} specifies the branching in terms of variable order and domain cut/assignment policy. 
To avoid future useless branching below the current node in the search tree, each constraint has an associated \emph{propagator}, which dynamically removes domain values that cannot be part of a solution to that constraint. 
Depending on the propagators used, domain reduction of constraints can be more or less effective: the notion of \emph{consistency} characterizes 
propagator effectiveness. A propagator that only keeps domain values that participate in a solution to its constraint achieves \emph{generalized arc-consistency} (GAC).
A solution to a constraint is obtained when the domains of all its variables are singletons.
Any solution to the problem must satisfy all the stated constraints, including domain constraints. Optimization problems are modeled 
through a specific variable that is minimized or maximized, subject to an objective constraint.
\subsection{Graph isomorhism solvers}
A broad literature exists on constraint-based techniques for solving subgraph isomorphism problems~\cite{DBLP:conf/aaai/Regin94,zan2010,DBLP:journals/ai/Solnon10,DBLP:conf/cp/AudemardLMGP14,DBLP:conf/cpaior/ArchibaldDHMP019}. The subgraph isomorphism problem is to determine whether an injective mapping exists from one given graph to another, such that adjacent pairs of vertices are mapped to adjacent pairs of vertices. 
As far as we know, a powerful state-of-the-art dedicated solver for graph isomorphism is the Glasgow Subgraph Solver. It is based on constraint programming~\cite{DBLP:conf/gg/McCreeshP020}.  
Most of the best alternatives are also based on constraint programming~\cite{DBLP:conf/gbrpr/Solnon19},
which motivated us to use this technology. 
However, the problem tackled in this paper is a weighted assignment problem correlated with a subgraph isomorphic constraint on trees plus, eventually, specific constraints for the objective function. As demonstrated in Section~\ref{sec:expes}, considering trees, a central feature of the fragment-based approach, allows us to use a general constraint programming solver while keeping a satisfactory solving process. We may exploit, in the future, its problem-independent modeling features for fast prototyping of further variants of the problem, which is an essential advantage. This problem has some conceptual links with the maximum common connected subgraph problem~\cite{DBLP:conf/aaai/HoffmannMR17}.
\section{Problem description}~\label{sec:pb}
This section formally describes the problem. We consider two molecules decomposed into fragments: the query molecule and the target molecule. An acyclic, undirected graph represents each molecule. We search for an isomorphic mapping of a connected subgraph of the query to the target minimizing a sum of costs.  
A subgraph is, in our context, a graph that may have only part of the nodes and edges of the original query graph.

\subsection{Definitions}
\begin{definition}[Query graph]
The \emph{query} $Q = (N_Q, E_Q)$ is an acyclic, undirected, connected graph with $n_Q = |N_Q|$ nodes. The nodes are fragments represented as all distinct indexes in $\{0, 1, \ldots, n_Q-1\}$. 
\end{definition}
\begin{definition}[Target graph]
The \emph{target} $T = (N_T, E_T)$, an acyclic, undirected, connected graph with $n_T = |N_T|$ nodes, $n_T \geq n_Q$.  Index $0$ is reserved for modeling a ``dummy target". Therefore, the graph has an isolated vertex, $0$, and a connected component of linked fragments $\{1, \ldots, n_T-1\}$.
\end{definition}
In the graphs $T$ and $Q$, between two nodes $i$ and $j$, $(i,j)$ and $(j,i)$ represent the same edge. 
\begin{definition}[Threshold]
 $\delta$ is a positive user-defined integer threshold. This parameter is used to charaterize acceptable associations of fragment pairs. 
\end{definition}
\begin{definition}[Score matrix]
$S$ is a $n_Q \times n_T$ matrix of fragment-pair integer scores, for potential assignments between the query and the target molecule. Scores range in $[1,\mathit{max_S}]$ for target indexes in $\{1, \ldots, n_T-1\}$. The first column is the ``dummy target": all cells equal to 0. 
\end{definition}
\begin{definition}[Number of query fragments]
$\mathit{nlink}$ is the number of query fragments that must be asociated with target fragments in $\{1, \ldots, n_T-1\}$. We must have $0<\mathit{nlink}\leq n_Q$.  
\end{definition}

\begin{figure*}[!h]
\begin{center}
~\\
\scalebox{0.7}{
\begin{tabular}{l@{ }l@{ }l@{ }l@{ }l@{ }l@{ }l@{ }l@{ }l@{ }l@{ }l@{ }l@{ }l@{ }l@{ }l@{ }l@{ }l@{ }l@{ }l@{ }l@{ }l@{ }l@{ }l@{ }l@{ }l@{ }l@{ }l}
\textbf{~~~} 
&\textbf{0}&\textbf{1}&\textbf{2}&\textbf{3}&\textbf{4}&\textbf{5}&\textbf{6}&\textbf{7}&\textbf{8}&\textbf{9}&\textbf{10}&\textbf{11}&\textbf{12}&\textbf{13}&\textbf{14}&\textbf{15}&\textbf{16}&\textbf{17}&\textbf{18}&\textbf{19}&\textbf{20}&\textbf{21}&\textbf{22}&\textbf{23}&\textbf{24}&\textbf{25}\\
 ~~~\textbf{0}&0 & 26 & 45 & 57 & 85 & 15 & 92 & 71 & 74 & 91 & 23 & 81 & 5 & 12 & {\color{blue}\textit{73}} & 15 & 31 & 91 & 78 & 53 & 39 & 82 & 17 & 47 & 17 & 25\\  
 ~~~\textbf{1}&{\color{blue}\textit{0}} & 92 & 30 & 89 & 84 & 70 & 49 & 85 & 44 & 25 & 67 & 68 & 35 & 44 & 75 & 46 & 38 & 26 & 18 & 89 & 18 & 25 & 100 & 14 & 18 & 15\\ 
 ~~~\textbf{2}&0 & 96 & 33 & 65 & 59 & 14 & 79 & 56 & {\color{blue}\textit{13}} & 7 & 17 & 35 & 24 & 92 & 62 & 17 & 13 & 40 & 86 & 89 & 48 & 56 & 95 & 29 & 50 & 96\\ 
 ~~~\textbf{3}&{\color{blue}\textit{0}} & 42 & 42 & 78 & 72 & 86 & 43 & 94 & 72 & 12 & 91 & 43 & 58 & 46 & 68 & 10 & 27 & 39 & 54 & 89 & 45 & 78 & 61 & 80 & 56 & 99\\ 
 ~~~\textbf{4}&0 & 45 & 13 & 63 & 63 & 88 & 8 & 98 & 39 & 26 & 1 & 98 & 22 & 56 & 32 & 77 & 44 & {\color{blue}\textit{9}} & 73 & 53 & 48 & 59 & 25 & 84 & 84 & 15\\ 
 ~~~\textbf{5}&0 & 26 & 43 & 4 & 81 & 36 & 91 & {\color{blue}\textit{4}} & 34 & 22 & 84 & 46 & 27 & 30 & 25 & 4 & 8 & 13 & 28 & 76 & 10 & 43 & 64 & 10 & 91 & 39\\ 
 ~~~\textbf{6}&0 & 57 & {\color{blue}\textit{47}} & 56 & 16 & 61 & 10 & 31 & 74 & 25 & 81 & 27 & 68 & 90 & 3 & 1 & 96 & 79 & 34 & 19 & 74 & 96 & 9 & 100 & 70 & 84\\ 
 ~~~\textbf{7}&{\color{blue}\textit{0}} & 58 & 14 & 64 & 52 & 57 & 96 & 70 & 24 & 64 & 95 & 1 & 19 & 15 & 77 & 94 & 25 & 17 & 61 & 80 & 53 & 75 & 26 & 66 & 79 & 81\\  
 ~~~\textbf{8}&0 & 82 & 90 & 91 & 58 & 51 & 26 & 76 & 5 & 99 & 20 & 86 & 13 & 21 & 44 & 47 & {\color{blue}\textit{11}} & 84 & 2 & 44 & 95 & 75 & 9 & 25 & 24 & 6\\ 
 ~~~\textbf{9}&0 & 56 & 35 & 61 & 17 & 93 & 30 & 80 & 18 & 20 & 59 & 76 & 13 & 38 & 10 & 92 & 41 & 61 & {\color{blue}\textit{1}} & 87 & 66 & 65 & 10 & 51 & 63 & 65\\  
 ~~~\textbf{10}&{\color{blue}\textit{0}} & 12 & 16 & 34 & 88 & 10 & 30 & 78 & 84 & 20 & 99 & 53 & 53 & 67 & 63 & 74 & 6 & 75 & 10 & 54 & 57 & 83 & 93 & 24 & 65 & 89\\ 
 ~~~\textbf{11}&0 & 62 & 3 & 12 & 25 & 97 & {\color{blue}\textit{17}} & 68 & 86 & 72 & 9 & 23 & 32 & 22 & 41 & 80 & 4 & 84 & 30 & 73 & 53 & 42 & 38 & 10 & 72 & 61\\  
 ~~~\textbf{12}&{\color{blue}\textit{0}} & 81 & 44 & 71 & 30 & 28 & 58 & 79 & 50 & 55 & 66 & 76 & 61 & 20 & 21 & 74 & 66 & 21 & 25 & 51 & 29 & 31 & 20 & 53 & 16 & 90\\  
 ~~~\textbf{13}&{\color{blue}\textit{0}} & 9 & 89 & 44 & 70 & 47 & 37 & 20 & 95 & 38 & 94 & 20 & 21 & 42 & 55 & 27 & 70 & 18 & 31 & 59 & 54 & 58 & 37 & 82 & 30 & 43\\  
 ~~~\textbf{14}&{\color{blue}\textit{0}} & 11 & 27 & 81 & 56 & 68 & 43 & 35 & 10 & 12 & 15 & 23 & 64 & 70 & 19 & 60 & 44 & 88 & 67 & 36 & 39 & 43 & 34 & 14 & 45 & 90\\  
 ~~~\textbf{15}&{\color{blue}\textit{0}} & 11 & 38 & 46 & 62 & 61 & 83 & 4 & 32 & 28 & 45 & 68 & 23 & 58 & 9 & 48 & 48 & 45 & 27 & 43 & 8 & 58 & 67 & 77 & 45 & 90\\ 
 ~~~\textbf{16}&0 & 35 & 51 & 93 & {\color{blue}\textit{6}} & 57 & 11 & 80 & 88 & 88 & 49 & 50 & 82 & 61 & 84 & 34 & 99 & 77 & 16 & 17 & 24 & 2 & 30 & 58 & 38 & 88\\  
 ~~~\textbf{17}&{\color{blue}\textit{0}} & 91 & 54 & 33 & 45 & 71 & 57 & 11 & 20 & 84 & 25 & 15 & 75 & 11 & 95 & 71 & 92 & 34 & 8 & 5 & 56 & 40 & 84 & 26 & 60 & 56\\  
 ~~~\textbf{18}&{\color{blue}\textit{0}} & 51 & 94 & 48 & 61 & 86 & 60 & 49 & 12 & 5 & 1 & 21 & 86 & 91 & 60 & 82 & 88 & 62 & 89 & 6 & 57 & 33 & 100 & 59 & 91 & 4\\  
 ~~~\textbf{19}&{\color{blue}\textit{0}} & 47 & 83 & 36 & 20 & 79 & 34 & 92 & 48 & 86 & 79 & 4 & 68 & 97 & 78 & 61 & 84 & 33 & 84 & 55 & 57 & 58 & 82 & 50 & 36 & 57\\	
\end{tabular}
}
\end{center}
\begin{center}
\includegraphics[width=5.2in]{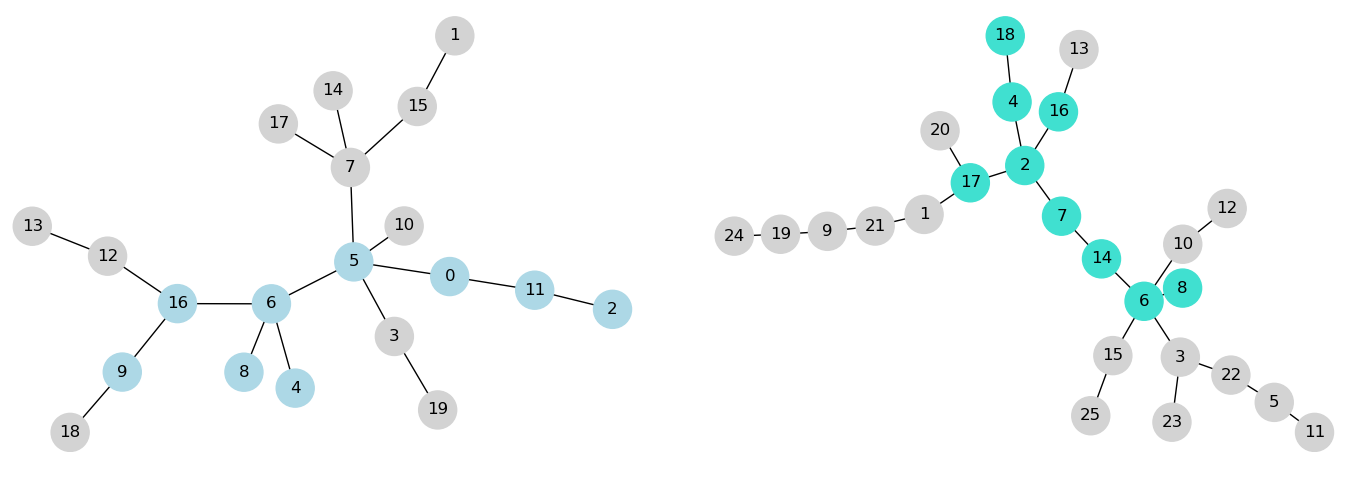}
\caption{On top, a score matrix $S$ of a 20 fragments query and 25 fragments target (this example uses a discretized matrix). We added node 0 to the target to represent a ``dummy " fragment. Below, a query graph $Q$ (on the left) and a target graph $T$ (dummy node 0 is omitted), associated with $S$. The following solution for $\mathit{nlink}=9$ is optimal ($T$ in blue $\rightarrow Q$ in turkoise): $0 \rightarrow 14$; $2 \rightarrow 8$;  $4 \rightarrow 17$; $5 \rightarrow 7$; $6 \rightarrow 2$; $8 \rightarrow 16$; $9 \rightarrow 18$; $11 \rightarrow 6$; $16 \rightarrow 4$;. Its cost is 181 (see the blue values in $S$: 73+13+9+4+47+11+1+17+6). } \label{fig:graphs}
\end{center}
\end{figure*}

\subsection{Feasible and optimal solutions}~\label{sec:solgen}
A feasible solution to the problem is a one-to-one mapping of $\mathit{nlink}$ distinct nodes in the query graph onto $\mathit{nlink}$ distinct nodes in the target graph, such that:
\begin{itemize}
\item The $\mathit{nlink}$ selected nodes in $N_Q$ form a connected component of $Q$.
\item If two selected nodes in $N_Q$ are linked by an edge in $E_Q$ then their two associated nodes in $N_T$ are also linked by an edge in $E_T$.
\item Each pair of associated fragments has a \emph{score} in $S$ that must be strictly less than the threshold $\delta$.
\item The \emph{cost} of the solution is the sum of scores in $S$ of all the selected pairs of query and target nodes.
\end{itemize}
A solution is \emph{optimal} if there exists no feasible solution with a strictly lower cost. 
Figure~\ref{fig:graphs} shows an example of matrix $S$ and an optimal solution for $\mathit{nlink}=9$. When the goal is to generate multiple, diverse solutions, we look for near-optimal cost solutions. 

Beyond a simple mapping between a query and a target, we consider three solution generation procedures. They make sense in the context of molecular databases because they correspond to standard end-users' queries. 
\begin{enumerate}
\item Find, should they exist, all optimal solutions for $\mathit{nlink} = 1$ to $\mathit{nlink} = n_Q$.
\item Given an integer $\mathit{nlink} \in [1, n_Q]$, look for a set of diverse near-optimal  solutions, including the optimal one. Solution 
quality is expressed by a distance to the optimal cost. We discuss in section~\ref{sec:CP} how this distance is computed and the diversity measured.
\item Find, should it exists, an optimal solution for $\mathit{nlink} \in [1, n_Q]$, where $n \leq \mathit{nlink}$ associations have been fixed manually. 
\end{enumerate}
It is straightforward that the third procedure is a special case of the first one. 
\section{Constraint Model}\label{sec:CP}
This section details the constraint programming model we designed for the problem. It can be parameterized for use with any of the solution generation schemes described in Section~\ref{sec:solgen}.
\subsection{Variables}
\begin{itemize}
\item $X_Q$, a set of $n_Q$ variables for the query. For all $x_Q^i \in X_Q$, the assigned value in a solution should be a valid target index, \emph{i.e.}, 
\[D(x_Q^i)= \{j \in \{0, 1, \ldots, n_T\}: S[i][j]<\delta\}.\] 
Recall that value $0$ represents a ``dummy" target, useful to state that some query fragments are not concretely mapped onto the target molecule.  
\item $X_S$, a set of $n_Q$ variables one-to-one mapped with $X_Q$. For all $x_S^i \in X_S$, the domain is the set of possible scores 
in the row $i$ of $S$, \emph{i.e}, 
\[D(x_S^i) =  \{s \in S[i]: s<\delta\}.\]
\item $\mathit{occ_0}$ is a fixed integer variable representing the number of variables in $X_Q$ assigned with value 0 (\emph{i.e.}, the query fragments assigned to the dummy target). 
\[D(\mathit{occ_0}) = \{n_Q-\mathit{nlink}\}.\]
\item $\mathit{obj}$ is the objective variable to be minimized, constrained to be equal to the sum of variables in $X_S$. 
\[D(obj)=[0,\mathit{nlink}\times(\delta-1)].\] 
\end{itemize}
A supplementary variable should be stated for generating multiple, diverse solutions. Its domain and use are described in Section~\ref{sec:div}. 
\subsection{Constraints}\label{sub:constraints}
\subsubsection{Assignment and objective function}
A solution should satisfy three assignment constraints and the objective constraint. 
\begingroup
\renewcommand\labelenumi{(\theenumi)}
\begin{enumerate}
\item If $x_Q^i \in X_Q$ is assigned value $j$, the score variable 
$x_S^i \in X_S$ must be equal to $S[i][j]$. 
\item The variables in $X_Q$ with a strictly positive value should be all different: each query fragment is assigned to a distinct target fragment.
\item The number of variables in $X_Q$ with value $0$ should be equal to $\mathit{occ_0}$.  
\item Variable $obj$ should be equal to the sum of the values assigned to the variables in $X_S$.
\end{enumerate}
\endgroup
Constraint (1) can be encoded through table constraints on all variable pairs $[x_Q^i,x_S^i]$, $i \in [0, n_Q[$, \emph{i.e.}, by explicitly listing the set of allowed combinations of values for each variable pair $(x_Q^i,x_S^i)$ from the score matrix $S$. 
State-of-the-art constraint programming solvers associate table constraints with GAC propagators. 

Constraint (2) is \texttt{AllDifferentExcept0}. This constraint is derived from the well-known \texttt{AllDifferent} constraint~\cite{DBLP:conf/aaai/Regin94} and comes up with a GAC algorithm. 
\begin{definition} {\normalfont\texttt{AllDifferentExcept0}} holds on a variable set $X$. It enforces all the variables in $X$ to take distinct values, except those variables that are assigned value 0.
\end{definition}
In our model, \texttt{AllDifferentExcept0} is enforced on $X_Q$.

Constraint (3) is \texttt{Count}~\cite{sicstus}.
\begin{definition} {\normalfont\texttt{Count}} takes an integer value $v$ as argument and holds on a variable set $X$ and a variable $\mathit{occ}$. It enforces the number of variables in $X$ taking value $v$ to be equal to $\mathit{occ}$. 
\end{definition}
In our model, we use \texttt{Count} with the three arguments 0, $X_Q$ and $\mathit{occ_0}$. 

The objective constraint (4) is a sum constraint: 
\[\mathit{obj} = \sum_{x_S^i \in X_S} x_S^i.\] 
\subsubsection{Isomorphism}
A solution to the problem should satisfy two connectivity constraints.
\begingroup
\renewcommand\labelenumi{(\theenumi)}
\begin{enumerate}
\item If two query fragments are linked by an edge (a chemical bond), the assigned target fragments should also be linked by an edge, except if at least one of the two target fragments is the dummy value $0$. 
\item All the variables in $X_Q$ that take a strictly positive value should belong to the same connected component in the graph $Q$. Observe that enforcing simultaneously this constraint and Constraint (1) guarantees that the target will also be a single connected component, which is required.  
\end{enumerate}
\endgroup
Constraint (1) can be encoded through table constraints on all variable pairs $[x_Q^{i1},x_Q^{i2}]$, $i_1 \in [0, n_Q[$, $i_2 \in [0, n_Q[$, where allowed combinations of values are explicitly set. This table ensures that any two variables $x_Q^{i1} \in X_Q$ and $x_Q^{i2} \in X_Q$ such that $(i_1,i_2) \in E_Q$ either are assigned two values $j_1$ for $x_Q^{i1}$ and $j_2$ for  $x_Q^{i1}$ such that $(j_1, j_2) \in E_T$, or are such that at least one of the two takes value 0. 
%

To encode Constraint (2) we can exploit the property of having $Q$ acyclic and connected. 
\begin{prop}\label{prop1}
There exists exactly one path between any two distinct nodes $i_1$ and $i_2$ in $Q$. 
\end{prop}
\begin{proof} $Q$ has a single connected component so at least one path exists between any two distinct nodes. $Q$ is acyclic so at most one path exists between any two distinct nodes. 
\end{proof}
From property~\ref{prop1}, we can generate the path between any pair of query node variables $(x_Q^{i1} \in X_Q, x_Q^{i2} \in X_Q)$. If $x_Q^{i1}>0$ and $x_Q^{i2}>0$ then all the other variables in the path must take a value strictly greater than $0$. We can state a logical expression that involves the \texttt{Count} constraint to forbid the value 0 for the variables in the path when $x_Q^{i1}>0$ and $x_Q^{i2}>0$. The pseudo-code is the following.  \\\\
\texttt{for i1 in [0,}$n_q$\texttt{-1[:}\\
\indent\texttt{for i2 in [i1+1,}$n_q$\texttt{[:}\\
\indent\indent\texttt{vars = variables of }$X_Q$\texttt{ in the path from i1 to i2 (excluding }$x_Q^{i1}$\texttt{ and }$x_Q^{i2}$\texttt{)}\\ 
\indent\indent\texttt{if vars.length > 0:}\\
\indent\indent\indent\texttt{model.addConstraint((}$x_Q^{i1}$\texttt{ == 0 or }$x_Q^{i2}$\texttt{ == 0) or count(0,vars,0))}
\subsection{Search strategy}\label{sec:strategy}
We use the following search strategy for solving this model. 
\begin{enumerate}
\item Branch first on the variables in $X_Q \cup X_S$ using the default search strategy of the solver, ``domOverWdeg"~\cite{DBLP:conf/ecai/BoussemartHLS04}, and the minimum value in each domain, enhanced with a last-conflict reasoning~\cite{DBLP:journals/ai/LecoutreSTV09}.
\item Assign $\mathit{occ}_0$ and $\mathit{obj}$. 
\end{enumerate}
\subsection{Solution diversity}\label{sec:div}
To generate $k$ maximally-diverse solutions to an instance that involves $n$ variables, we state a constraint that represents some function:
\[ f:(\mathbb{Z}^n)^k\mapsto~\mathbb{Z}_+.\]
To encode this constraint, we use the GAC propagator suggested by Hebrard et al.~\cite{heb07b}, based on the Hamming Distance. The definition below accepts Hamming as well as other distance measures~\cite{pettra19}. It implicitly assumes that solution diversity is likely to be maximized. 
\newcommand{\diver}{\mathit{y_d}}
\newcommand{\prev}{\mathit{c}}
\begin{definition}\label{def:DiverseSum}
Let $\prev \geq 0$ be a constant, $X = \{x_0, \ldots, x_{n-1}\}$ be a set of variables, $\diver$ a variable, ${\mathcal A}$ a set of previous $k-1$ assignments of $X$ and $\delta: (\mathbb{Z}^n)^2\mapsto~\mathbb{Z}_+$ be a distance measure, such as the Hamming distance, 
based on pairwise comparisons of the values taken by each variable in $X$, that is, for any solution $A_j \in {\mathcal A}$,
\begin{equation}\label{eq:diverse_sum}
\delta(X, A_j) = \sum_{i = 0}^{n-1} \delta^x(A_j[i],x_i).
\end{equation}

{\normalfont\texttt{Diversity}}$(X, \diver, {\mathcal A},\delta,\prev)$ is satisfied if and only if:
\begin{equation}\label{eq:diverse_sum_iff}
\diver \leq \sum_{j=1}^{|{\mathcal A}|} \delta(X, A_j) + \prev.
\end{equation}
\end{definition}
The pseudo code for generating up to $k$ diverse solutions at a maximum distance of $\mathit{gap}$ to the optimal objective value is the following: \\\\
\texttt{model} $=$ \texttt{new model}\\
\texttt{$\mathit{sol} =$  solution that minimizes $\mathit{obj}$, if it exists}\\
\texttt{if $\mathit{sol} \neq \emptyset$:}\\
\indent\texttt{$\mathcal{A}$ = $k \times n_q $ empty matrix} \\
\indent\texttt{$\mathcal O$ = list of length $k$}\\
\indent\texttt{$\mathcal{A}$[0] = $X_Q$ assignment in $\mathit{sol}$}\\
\indent\texttt{$\mathcal{O}$[0] = value assigned to $\mathit{obj}$ in $\mathit{sol}$}\\
\indent\texttt{$\prev$ = 0} \\
\indent\texttt{i = 1} \\
\indent\texttt{while i$<k$ and $\mathcal{A}$[i-1] is not empty:}\\
\indent\indent\indent\texttt{model = new model}\\
\indent\indent\indent\texttt{$D(\diver)$ = [0, $n_Q + \prev$]} \\
\indent\indent\indent\texttt{$D(\mathit{obj})$ = [0, $\mathcal{O}$[0] + $\mathit{gap}$]} \\ 
\indent\indent\indent\texttt{model.addConstraint(\texttt{Diversity}}$(X_Q, \diver, {\mathcal A},\delta,\prev))$\\
\indent\indent\indent\texttt{$\mathit{sol} =$  solution that maximizes $\diver$, if it exists}\\
\indent\indent\indent\texttt{if $\mathit{sol} \neq \emptyset$:}\\
\indent\indent\indent\indent\texttt{$\mathcal{A}$[i] = $X_Q$ assignment in $\mathit{sol}$}\\
\indent\indent\indent\indent\texttt{$\mathcal{O}$[i] = value assigned to $\mathit{obj}$ in $\mathit{sol}$} \\  
\indent\indent\indent\indent\texttt{$\prev$ = $\prev$ + value of $\diver$ for $\mathcal{A}$[i]}\\
\indent\indent\indent\texttt{i = i+1}\\

The diversity variable $\diver$ must be considered in the search strategy (assigned after $\mathit{occ}$ and $\mathit{obj}$ in the second step of Section~\ref{sec:strategy}). 

\section{Experiments}\label{sec:expes}
We used a Microsoft Windows 10 computer with a I5-8350U processor and 16GB of RAM, using the 
Choco 4 Java contraint programming solver (see: \url{https://choco-solver.org/}). The main objective of these experiments is determining   
whether constraint programming technology can be used in the context of querying large molecular databases using fragment-based signatures.  
Further integration of this model might require implementing the model with a different programming language.\footnote{A general purpose language, like Java, C$++$, etc., i.e., not dedicated to constraint programming or to a specific optimization tool.}

\subsection{Real data}\label{sec:real}
We considered $n = 76$ molecules from the Zinc database (\url{https://zinc.docking.org/}), each fragmented with less than 8 fragments.  
\begin{figure*}[!h]
\begin{center}
\includegraphics[width=1.5in]{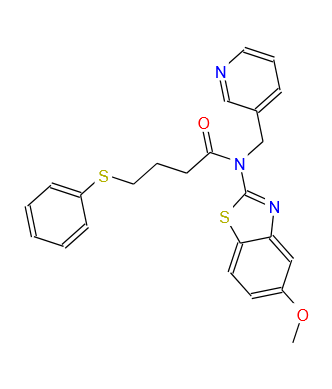}\\
\includegraphics[width=3.6in]{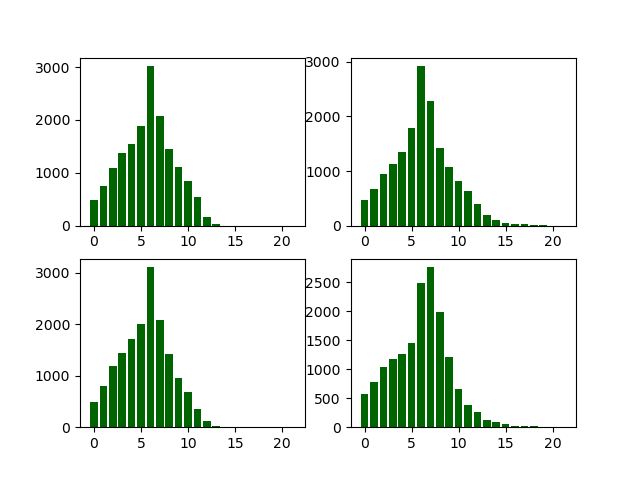}
\caption{Shapes of the four individual fragments of \texttt{ZINCID:10020366}. } \label{fig:shapes}
\end{center}
\end{figure*}
All individual fragments were ray traced, as well as predefined fusions (fixed connected fragment sets), which are taken into account in the current software to search for the best matches. 
The use of fusions is out of the scope of this article. The resulting shapes are represented as 32-bin histograms, with values ranging from 0 to a few thousand units. Figure~\ref{fig:shapes} shows the example of \texttt{ZINCID:10020366}.

\begin{figure*}[!h]
\begin{center}
\includegraphics[width=4in]{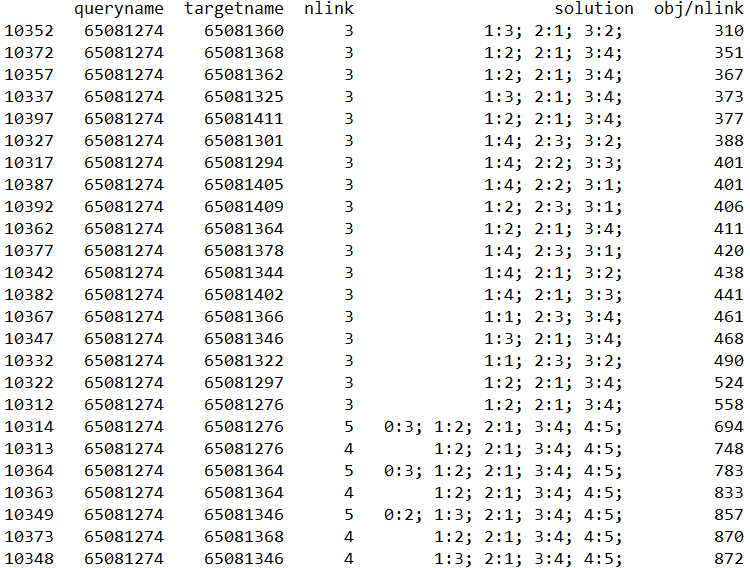}~\\~
\includegraphics[width=1.2in]{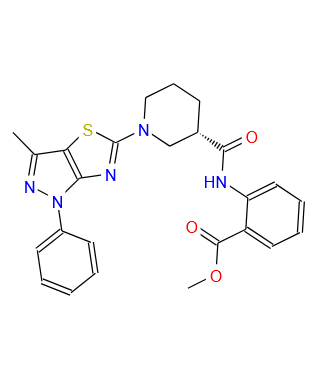}
\includegraphics[width=1.2in]{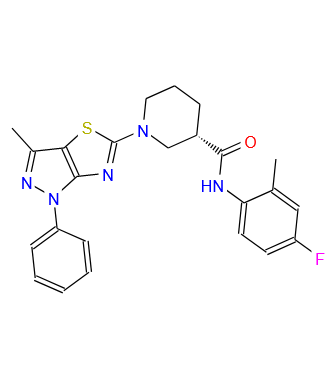}
\includegraphics[width=1.2in]{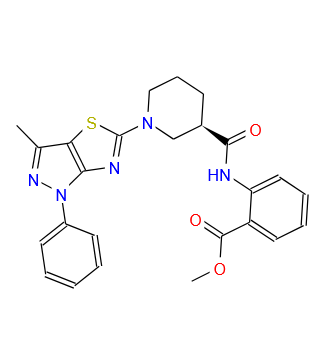}
\includegraphics[width=1.2in]{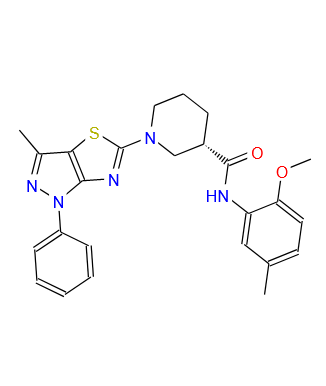}
\caption{On top, a sample of the output for a specific query (\texttt{ZINCID:65081274}),  with 25 optimal solutions such that $\mathit{nlink}$ is between 3 and 5 fragments, sorted by increasig value of $\mathit{obj}/\mathit{nlink}$. Below, from left to right: A drawing of the query \texttt{ZINCID:65081274}, decomposed into $n_Q=5$ fragments,  the target with minimum $\mathit{obj}/\mathit{nlink}$ value among all optimal solutions for $\mathit{nlink}=3$ (\texttt{ZINCID:65081360}), and two first targets using the comparator $\mathit{obj}/\mathit{nlink}$ when $\mathit{nlink}=n_Q=5$ (\texttt{ZINCID:65081276} and \texttt{ZINCID:65081364}). } \label{fig:zincexample}
\end{center}
\end{figure*}
\subsubsection{Data processing}
To compare two histograms we use the Manhattan distance. This metric is fast and was proved relevant to compound search (see~\cite{zau13}). Let $h_i^k$ denote the height of bin $k$ for the histogram $i$, each score is computed as follows:
\[S[i][j] = \sum_{\mathit{k=0}}^{\mathit{k = 31}} |h_i^k-h_j^k|.\] 

We generated the 2840 score matrices of molecules pairs. For generating the solutions, we considered for each pair that the query is the smaller of the two molecules in terms of number of fragments (we took the oder in the datafile in case of equality). The generation process takes less than one second. 

\subsubsection{Results}
For each of the 2840 instances, our model found the optimal solutions with $\mathit{nlink}$ ranging from 1  
to the number of fragments of the query, i.e., $\mathit{nlink}=n_Q$, should a solution exists.

\begin{table}[!h]
\begin{center}
\begin{tabular}{|c|c|c|} 
\hline
&  \# Bactracks & Time (s) \\
 \hline
Mean &  7.19  &   0\\
Standard deviation & 5.8 &    0\\
Min  &  0& 0 \\
Median &7& 0\\
Max & 53 & 0.02\\
 \hline
\end{tabular}
\end{center}
\caption{Number of backtracks and time for proving optimal solution or absence of solution on 12269 problems extracted from the ZINC database. A total of 11706 instances had a solution. All instances were proved optimal in less than 53 backtracks.}\label{tabledata}
\end{table}
 
A total number of 
12269 problems were solved, in which 11706 had a solution. Figure~\ref{fig:zincexample} shows an example of output. Table~\ref{tabledata} shows that the optimal solutions are found in almost real-time on this dataset, making the approach suited to large database searches.  

\subsubsection{Discussion}

We generated all the score matrices before solving the instances and saved them to files, which were then read to populate the variable domains. Although the size of the query and target allowed this calculation and file reading to be performed in nearly constant time for each pair of molecules, we could not reasonably perform this for $n$ equal to ten million molecules.  We would have had to generate $\frac{n\times(n-1)}{2}$ files. This can be seen as a weakness regarding the idea of using an optimization approach rather than enumerating some fragment or fusion mappings and computing their scores. However, we can observe that the goal is to find targets for a specific subset of queries, possibly only one. If there are only 100 queries, precomputing matrices for millions of potential targets is feasible with appropriate hardware. Otherwise, one may consider creating them before each solve or performing basic data filtering before querying the database, e.g., 
restrict to specific sets of target molecules.   

In addition, it is a good idea to run the solver with a single query against multiple targets at once. While this idea requires further study, we believe that avoiding the cost of one solver run per instance should be beneficial in a real-time setting.  
\subsection{Simluated data}
Although the typical fragment size is between 3 and 10 in the data considered at this time in computer-assisted drug discovery projects, it is of interest to investigate for future uses whether the model can be used with larger queries and targets. More importantly, the goal of this experiment is to determine whether, in datasets where molecules are randomly generated (with respect to size and number of branches in fragment trees), one can use a Branch and Bound search to map a portion of the query fragments on the fly, given connectivity and dissimilarity constraints. The response time should be close to real time. Unlike most application papers, simulated data is at least as important as real data in our context. 

We generated random acyclic graphs and score matrices. We created 100 sets of instances with a query ranging from 5 to 25 nodes on a target of 50 nodes, with score values in [1,100] for the matrix $S$, and $\delta > 100$. For each instance, we ran the model to search for the optimal mapping with $\mathit{nlink}$ ranging from 1 to the query size. The total number of problems solved was 7500. All optimal solutions were found in case a solution exists. Otherwise, the solver proved that there was no solution. In total, 6855 instances had a solution. 
\begin{table}[!h]
\begin{center}
\begin{tabular}{|c|c|c|} 
\hline
&  \# Bactracks & Time (s) \\
 \hline
Mean &  1312.06  &   0.19\\
Standard deviation & 2058.30 &    0.32\\
Min  &  0& 0 \\
Median &485& 0.05\\
Max & 30696 & 3.46\\
 \hline
\end{tabular}
\end{center}
\caption{Number of backtracks and time for proving optimality or absence of solution on 7500 problems with a query of 5 to 25 nodes, a target of 50 nodes, and $\mathit{nlink}$ ranging from 1 to the query size. 6855 instances had a solution. 7066 instances were solved in less than 1000 backtracks.}\label{table1}
\end{table}

\begin{figure*}[!h]
\begin{center}
\includegraphics[width=6in]{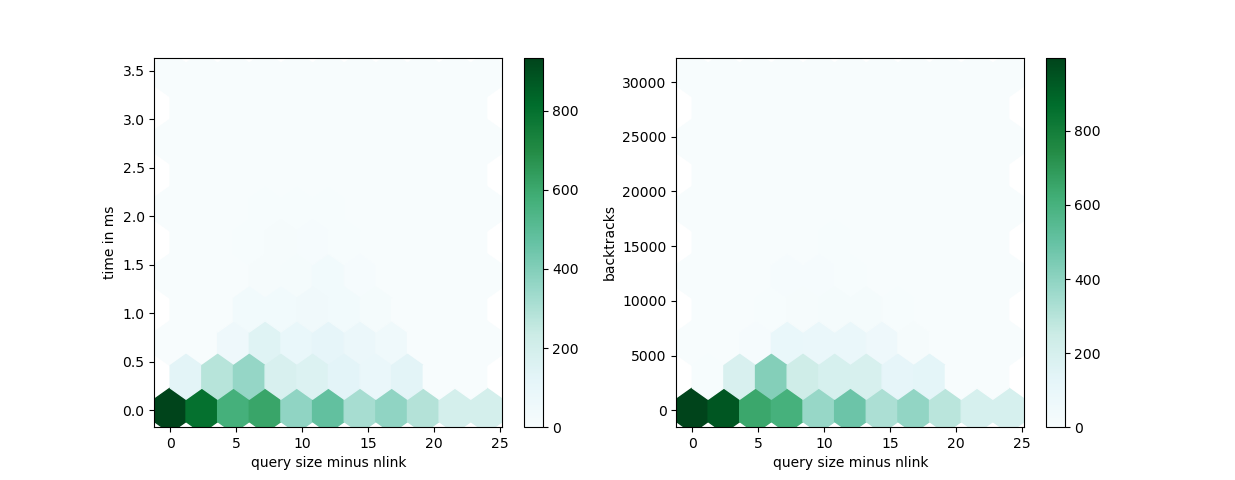}
\caption{Correlation between the difference $n_Q - \mathit{nlink}$ and, respectively, solving time and number of backtracks. } \label{fig:hexa}
\end{center}
\end{figure*}

Table~\ref{table1} and Figure~\ref{fig:hexa} show that the 7500 instances were solved within a time less than 3.46 seconds each, most in less than 0.5 seconds. We observed a few instances that were more difficult than the others, one of them requiring 30696 backtracks and 3.46 seconds to solve, which can be considered an extreme acceptable time for querying databases of millions of compounds, although a constant factor can be gained by using more powerful computers. 

\begin{table*}[!t]
\begin{center}
\begin{tabular}{|c|c|c|c|c|c|c|c|c|} 
\hline
Query & Target &  Score & \# Inst. & \# Inst. & \# Backtr. & \# Backtr. & Time & Time \\
size & size & range & & with a & Mean & Max & Mean & Max \\
& & & & solution & & & (s) & (s) \\
 \hline
5-25 & 30 & [0,100] & 7500 & 6373 & 720 & 10101 & 0.07 & 0.92 \\
5-25 & 70 & [0,100] & 7500 & 7004 & 1872 & 24997 & 0.44 & 5.56 \\
5-25 & 30 & [0,10000] & 7500 &  6393 & 755 & 10125  & 0.09 & 1.25  \\
5-25 & 50 & [0,10000] & 7500 & 6880 & 1395 & 23868  & 0.26 & 3.73 \\
5-25 & 70 & [0,10000] & 7500 & 7019  & 1940  & 30686  & 0.59   & 8.6  \\
\hline
\end{tabular}
\end{center}
\caption{Number of backtracks (mean and max) and time (mean and max) for proving optimal solution or absence of solution on problems with a query varying from 5 to 25 nodes and $\mathit{nlink}$ ranging from 1 to the query size.}\label{table1b}
\end{table*}

The results are very promising. Real-world problems will rarely involve queries of 25 fragments and targets as large as 50 fragments.  
Figure~\ref{fig:hexa} shows minor correlation between the difference $p = n_Q - \mathit{nlink}$ and the solving process. The instances were slightly harder to solve with $p \in [5,20]$.  
We did not observe significant correlation between the existence of a solution and solving time.  

We reproduced the same experiment with two different target sizes, respectively 30 fragments and 70 fragments, and the same experiment with scores ranging in [0,10000] and 
$\delta>10000$, with respectively 30, 50 and 70 fragments for the target. The results are summarized in Table~\ref{table1b}.  They are quite similar to the results of the original experiment.
\subsection{Multiple, diverse solutions}
We considered 1000 instances with targets of 50 nodes, query graphs of 15 nodes and $\mathit{nlink}=10$.  For each one, we generated the optimal solution plus 4 maximally diverse solutions with a maximal increase of 10\% on the objective value compared with the optimal.
\begin{table}[!h]
\begin{center}
\begin{tabular}{|c|c|c|} 
\hline
&  \# Bactracks & Time (s) \\
 \hline
Mean &  997.98  &   0.1\\
Standard deviation & 602.95 &    0.05\\
Min  &  6& 0 \\
Median &927& 0.1\\
Max & 4342 & 0.4\\
 \hline
\end{tabular}
\end{center}
\caption{Number of backtracks and time for finding the optimal solution and 4 maximally diverse solutions with $n_Q = 15, n_t=50$, and $\mathit{nlink} = 10$. The data statistics were computed on 1000 instances (5000 problems solved).}\label{table2div}
\end{table}

Generating maximally diverse solutions did not induce more backtracks or computational time compared with the first optimal solution (see Table~\ref{table2div}). In this benchmark, all the instances had solutions. We obtained similar results with other values for $\mathit{nlink}$ and $n_Q$. 

\begin{figure}[!h]
\begin{center}
\includegraphics[width=2.8in]{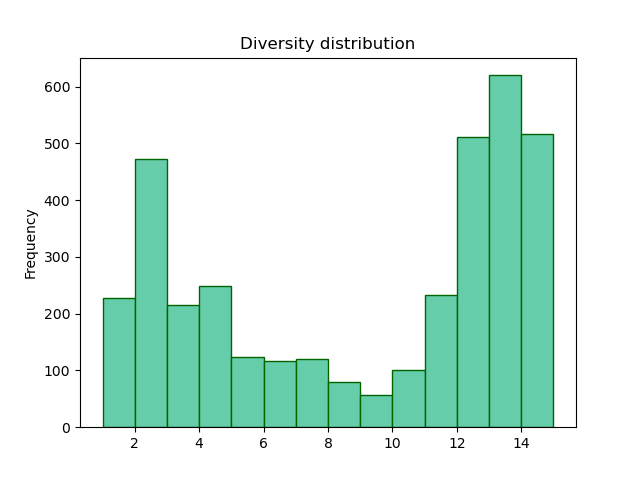}
\caption{Diversity value distribution of solutions generated after finding the optimal. } \label{fig:div}
\end{center}
\end{figure}

As we limited the increase in $obj$ to 10\%, it may occur that after a certain number of iterations on a given instance, the next solution found does not differ or only differs marginally from one of the previously generated solutions. However, the essential need is, in practice, to have at least one significantly diverse alternative to the optimal solution. We measured the distribution of Hamming distances in the second generated solution, just after the optimal one (see Figure~\ref{fig:div}). We observe that more than half of the solutions have eight or more variables, i.e., more than $n_Q/2$, that differ from the optimal solution. 

\section{Conclusion}
We have designed a constraint programming model to optimally solve variants of the fragment-based shape signature problem. The experimental results on real and simulated data are convincing.  The prototyping phase having fully satisfied our objectives, the next step will consist in integrating this component based on constraint programming to the existing tool. This step will essentially require an analysis of the hardware, software and data processing aspects, the model being adapted as it has been designed. 

This study opens new perspectives, such as the generation of multiple, diverse solutions and the consideration of larger molecules. A significant advantage of this approach is the simplicity of the model. This simplicity will likely allow us to adapt it in the future to other connectivity constraints, related to a particular application or to other recent techniques that use fragment-based shape signatures.
\bibliographystyle{named}
\bibliography{arxiv_tpetit}
\end{document}